\documentclass[conference]{IEEEtran}
\usepackage{longtable}
\IEEEoverridecommandlockouts
\usepackage{cite}
\usepackage{amsmath,amssymb,amsfonts}
\usepackage[linesnumbered,ruled,vlined]{algorithm2e}
\usepackage{algorithmic}
\usepackage{graphicx}
\usepackage{textcomp}
\usepackage{xcolor}
\usepackage{listings}
\usepackage{amsthm}
\usepackage{proof}
\usepackage{pgfplots}
\pgfplotsset{compat=1.17}

\usepackage{amsthm}
\usepackage[table]{xcolor}
\usepackage{comment}
\def\BibTeX{{\rm B\kern-.05em{\sc i\kern-.025em b}\kern-.08em
    T\kern-.1667em\lower.7ex\hbox{E}\kern-.125emX}}

\begin{document}

\newtheorem{definition}{Definition}
\newtheorem{theorem}{Theorem}
\newtheorem{lemma}[theorem]{Lemma}
\renewcommand{\qedsymbol}{} 

\title{
Improving Graph Embeddings in Machine Learning Using Knowledge Completion with Validation in a Case Study on COVID-19 Spread
}

\author{
\IEEEauthorblockN{Rosario Napoli}
\IEEEauthorblockA{\textit{University of Messina}\\
Messina, Italy \\
rnapoli@unime.it \\
0009-0006-2760-9889 
}
\and
\IEEEauthorblockN{Gabriele Morabito}
\IEEEauthorblockA{\textit{University of Messina}\\
Messina, Italy \\
gamorabito@unime.it\\
0009-0006-2144-8746} 
\and
\IEEEauthorblockN{Antonio Celesti}
\IEEEauthorblockA{\textit{University of Messina}\\
Messina, Italy \\
acelesti@unime.it\\
0000-0001-9003-6194} 
\and
\IEEEauthorblockN{Massimo Villari}
\IEEEauthorblockA{\textit{University of Messina}\\
Messina, Italy \\
mvillari@unime.it\\
0000-0001-9457-0677} 
\and
\IEEEauthorblockN{Maria Fazio}
\IEEEauthorblockA{\textit{University of Messina}\\
Messina, Italy \\
mfazio@unime.it\\
0000-0003-3574-1848} 
}

\maketitle

\begin{abstract}

The rise of graph-structured data has driven major advances in Graph Machine Learning (GML), where graph embeddings (GEs) map features from Knowledge Graphs (KGs) into vector spaces, enabling tasks like node classification and link prediction. However, since GEs are derived from explicit topology and features, they may miss crucial implicit knowledge hidden in seemingly sparse datasets, affecting graph structure and their representation. We propose a GML pipeline that integrates a Knowledge Completion (KC) phase to uncover latent dataset semantics before embedding generation. Focusing on transitive relations, we model hidden connections with decay-based inference functions, reshaping graph topology, with consequences on embedding dynamics and aggregation processes in GraphSAGE and Node2Vec.
Experiments show that our GML pipeline significantly alters the embedding space geometry, demonstrating that its introduction is not just a simple enrichment but a transformative step that redefines graph representation quality.
\end{abstract}


\section{Introduction}
In recent years, Graph Machine Learning (GML) has emerged as a rapidly evolving paradigm to exploit the relational nature of data, enabling new discoveries in several areas, such as social network analysis, recommendation systems and fraud detection \cite{b92}. By modeling data as graphs and focusing on the relationships between entities, GML offers capabilities that go beyond those of conventional Machine Learning (ML) approaches.
In fact, traditional ML models are typically designed to operate on tabular data structures, where each sample is treated as isolated from the others\cite{b72}.
As a result, they are not suitable for scenarios in which the connections between data carry important contextual meaning \cite{b6}. 
In particular, they overlook the importance of relationships between elements, which are essential to uncover meaningful structural patterns.
To address these limitations, GML methods incorporate graph-based representations, where entities (nodes) and their interactions (edges) are modeled explicitly.
At the same time, the growing availability of interconnected data has led to the widespread adoption of NoSQL Graph Databases (GDBs), which naturally support such representations. In particular, GDBs represent data as Knowledge Graphs (KGs), an evolution of graph data structure in which elements can have different labels and properties\cite{b12}. This natural representation describes real-world dynamics, enabling GML models to capture structural and semantic properties through node classification and link prediction\cite{b53}.

Central to GML are Graph Embedding (GE) techniques, which convert graph's structural and semantic information in a low-dimensional space that is given as input to ML models. These approaches have become increasingly popular due to their ability to handle relational data, which usually requires significant computational resources and task-dependent feature engineering \cite{b74}. 
Since GEs capture both local and global structure to provide features' vector representations\cite{b75}, a critical limitation of existing GDB-GML architectures is their dependence on explicit relationships in the dataset, without considering implicit or hidden connections. In particular, since KGs consist of nodes and edges with different semantics, their high heterogeneity causes information to remain unexpressed. As a result, a dataset that may initially appear fragmented, incomplete, or lacking in content, could actually contain a substantial amount of implicit knowledge that is not captured in the embedding space. This leads to representations that miss important structural and semantic patterns, overlooking potentially valuable hidden insights. 
Our key intuition is that introducing a dedicated Knowledge Completion (KC) step before the embedding phase can significantly improve the quality of data. 

In this study, we present a new GML pipeline where incomplete KGs are improved using a KC step that fills in missing links and node properties by using decay-based inference functions. 
Then, by working on a more informative representation of the data, where neighborhood structures and interaction pathways are completely changed, the embedding process can better capture meaningful patterns and interactions, ultimately redefining the quality of the learned representations. 
Moreover, we apply our proposed methodology to evaluate its impact on a real-world dataset. 
In particular, we assess how embeddings generated from the KC-enhanced pipeline differ in embedding space position when compared to those derived from the original, incomplete graph. This work demonstrates how explicitly restoring previously implicit relationships can lead to more accurate and semantically rich embeddings, confirming the value of incorporating a KC phase into the standard GML pipeline.

The rest of this paper is organized as follows.
Section II provides an overview of related work on GEs and the integration of KG. Section III then introduces the key definitions that form the basis of our approach. In Section IV, we describe the traditional GDB-GML pipelines, while in Section V, we present our approach, which includes an explicit KC phase. Section VI then details the experimental evaluation carried out on a real-world COVID-19 contact network. Finally, in Section VII we discuss conclusions and outline directions for future research.

\section{Related Works}
GEs have found applications in various domains, including: social network analysis, recommendation systems, bioinformatics and cybersecurity\cite{b79}.
Their ability to capture complex relationships and patterns in graph data makes them a valuable tool for solving a wide range of ML problems\cite{b80}, like in social network analysis, where they can be used to identify communities, predict links and classify users\cite{b81}.
However, the field is extremely new, and understanding how to enhance GEs for downstream ML tasks remains a crucial challenge. In particular, ongoing research in the field of GML has explored several strategies to improve both the quality and the expressiveness of GEs. These strategies range from the integration of additional semantic information into the graph structure to the design of more advanced neural architectures that can capture richer patterns.

Some approaches aim to enrich input graphs by incorporating knowledge from external sources, such as in KGs, where contextual and semantic attributes are added to both nodes and edges\cite{b84}. Other research efforts focus on developing sophisticated embedding models that are capable of encoding more complex structural and relational information, such as: transformer-based architectures, which allow for global attention mechanisms\cite{b85}, uncertainty-aware embeddings, that model confidence levels in the representation\cite{b86} and belief-propagation techniques\cite{b87}, which exploit probabilistic messages passing across the graph structure. Furthermore, recent models, such as BiGI\cite{b88} and SeHGNN\cite{b89},
propose novel training processes and hybrid architectures that emphasize the importance of global information or preaggregated multi-hop neighborhood features. These techniques reflect a growing trend toward capturing not only local connectivity patterns but also high-level and graph-wide dependencies, thus enabling more informative and generalizable embeddings.

Although numerous methodologies have been proposed to enhance the quality and effectiveness of GEs (with few related to KGs), none of these approaches explicitly focus on improving the graph structure prior to the embedding phase with the specific goal of changing their expressiveness. Moreover, there is a lack of methods that are agnostic to both the embedding phase and the hyperparameter tuning process.
This highlights the need for a KC stage that can effectively capture real dataset behavior and semantics before the embedding step. Such a completion step could not only improve the overall quality of the embeddings but also enhance ML downstream, as a more connected graph structure may allow for faster and more efficient information propagation across fewer steps.

To the best of our knowledge, this specific separation of concerns and the pre-processing strategy we propose have not been explored in prior literature, marking a significant innovation in the field.

\section{Definitions}
This section introduces the formal definitions of core concepts used throughout the work.


\begin{definition}[Knowledge Graph]
A Knowledge Graph is defined as a tuple \( KG = (V, R, E, \ell_V) \), where:
\begin{itemize}
    \item \( V \) is the set of entities (nodes);
    \item \( R \) is the set of relationship types;
    \item \( E \subseteq V \times R \times V \) is the set of labeled, directed edges (triplets);
    \item \( \ell_V : V \rightarrow 2^C \) is a labeling function that assigns one or more class labels from the set \( C \) to each node.
\end{itemize}

\end{definition}

\begin{definition}[KG atomic unit]
The atomic unit of a KG is a triple\cite{b32}, that consists of:
\begin{itemize}
    \item  labelled head node \( h \);
    \item a labelled tail node \( t \);
    \item a directed labelled edge \( r\) from \( h \) to \( t \).
\end{itemize}
\end{definition}
\begin{definition}[Graph Embedding]
Let \( KG = (V, E, R,  \ell_V) \) be a Knowledge Graph. A Graph Embedding (GE) is a function
\[
f: V \rightarrow \mathbb{R}^d,
\]
where \( d \ll |V| \), which maps each node \( v \in V \) to a vector \( f(v) \in \mathbb{R}^d \).
\end{definition}

\begin{definition}[Knowledge Completion]
Let \( KG = (V, E, R,  \ell_V) \) be a Knowledge Graph, Knowledge Completion is the task of predicting missing facts in the form of new atomic units \((h, t, r)\), where \(h, t \in V\) and \(r \in R\), which correspond to either add a new edge \((h, t)\) to \(E\), or to enrich an existing labelled edge by associating new properties or a new appropriate relation \(r \in R\).
\end{definition}


\section{Traditional GDB-GML Pipelines}
Traditional GDB-GML pipelines (Fig.\ref{fig:current_gml_pipeline}) typically follow a linear flow in which KGs are constructed by integrating data from various sources, and then directly used for downstream GML tasks. This process generally is carried out in two main steps:
\begin{enumerate}
    \item \textbf{Knowledge Fusion (KF)}. The process of integrating and unifying knowledge from multiple sources to create a coherent KG \cite{b33}.
    \item \textbf{Knowledge Reasoning (KR)}. The process of creating GEs and giving them as input for ML models.
\end{enumerate}
\begin{figure}[ht!]
    \centering
    \includegraphics[width=1\linewidth]{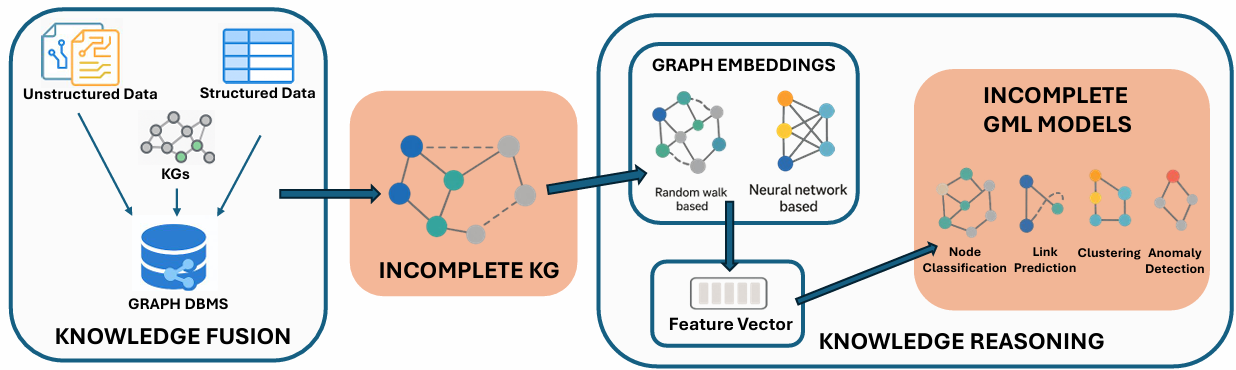}
    \caption{Current GML pipeline.}
    \label{fig:current_gml_pipeline}
\end{figure}
During the KR phase, GEs can be generated adopting several algorithms, which can be classified into two main categories \cite{b77}:
\begin{itemize}
\item  \textbf{Random walk based algorithms}. These algorithms rely on simulating random walks over the graph to capture the structural and contextual relationships between nodes. The core idea is to explore the graph through random sequences of nodes and then learn low-dimensional vector representations that preserve node co-occurrence patterns. A well-known example in this category is \textbf{Node2Vec}, which introduces a biased random walk strategy controlled by return and in-out parameters based on Depth First Search (DFS) and Breadth First Search (BFS), allowing it to efficiently capture both homophily (nodes with similar features) and structural equivalence (nodes with similar roles in the graph).
\item 
    \textbf{Neural network based algorithms}. These algorithms use deep learning architectures to aggregate and transform features from a node's local neighborhood to learn its embedding. They often rely on message-passing mechanisms, where each node iteratively updates its representation based on its neighbors' features. A representative model in this class is \textbf{GraphSAGE} (Graph Sample and Aggregation), which learns a function that generates embeddings by sampling and aggregating features from a node’s local neighborhood using techniques such as: mean, LSTM-based or pooling aggregators.
\end{itemize}

\section{KC-Enhanced Pipeline}
In many existing approaches, when KC is applied, it is included within the embedding process itself or it is implicitly integrated during model training. However, this often results in suboptimal representations, due to the incompleteness of the input KG.
In contrast, our approach (Fig. \ref{fig:improved_gml_pipeline}) restructures the traditional pipeline by explicitly introducing KC as a dedicated phase, specifically after KF and before KR.
By doing so, we aim to integrate all missing implicit edges related to a specific domain of knowledge, completing the graph's structure and then enhancing the semantic and structural quality of the KG before any reasoning or embedding.

\begin{figure}[ht!]
    \centering
    \includegraphics[width=1\linewidth]{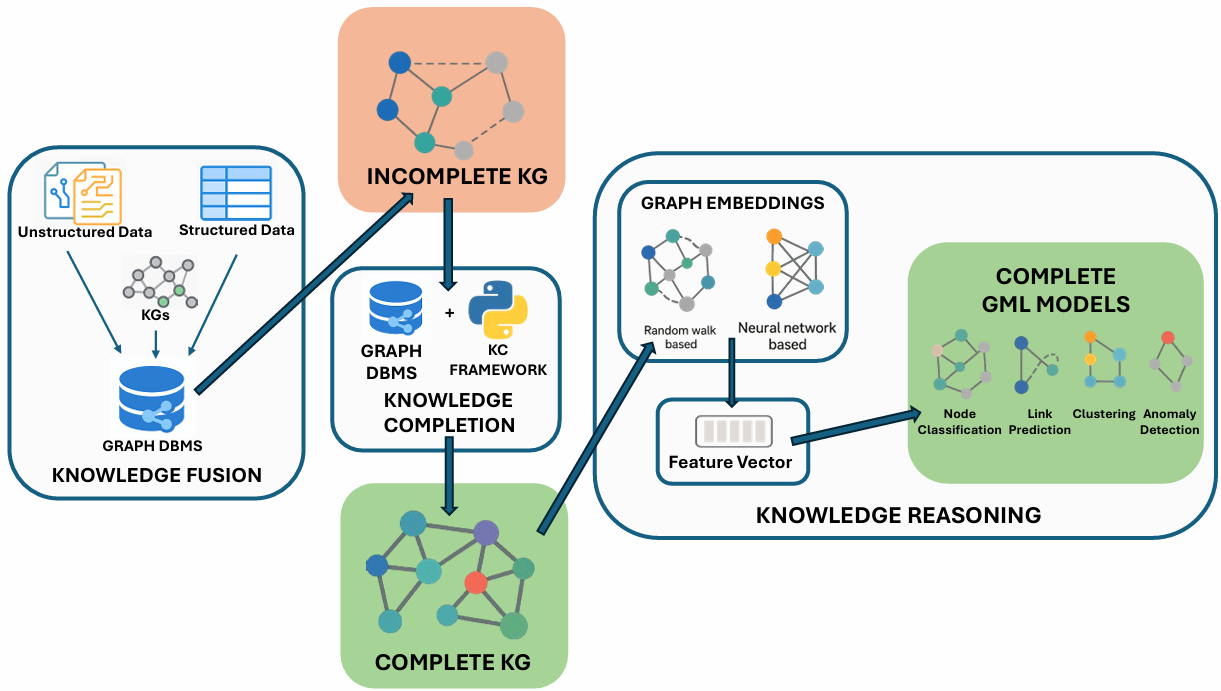}
    \caption{Our new GML pipeline.}
    \label{fig:improved_gml_pipeline}
\end{figure}
\subsection{KC through Scalable Transitive Relationship Inference}
At the heart of our KC phase there is the inference of scalable transitive relationships. These edges allow us to infer missing connections by using transitive patterns already present in the data. Unlike traditional KC techniques, that rely on statistical inference, our approach is based on a form of transitivity modeled by shared strength between connected nodes. Specifically, we introduce scalable transitive relationships as a mechanism to transform the graph structure, modifying both local and global connectivity patterns according to the intensity of the underlying connections. 
In the following, we give the formal definitions of these concepts.

\begin{definition}[Transitive Relationship]
In mathematics, a relation $R$ on a set $A$ is called transitive if, for any elements $a$, $b$ and $c$ in \( A \), the following holds:
\[
a R b \text{ and } b R c \Rightarrow a R c
\]
In other words, if \( a \) is related to \( b \) and \( b \) is related to \( c \), then \( a \) must also be related to \( c \).
\end{definition}

\begin{definition}[Transitive relationship in Knowledge Graphs]
    A relationship \( r \in R \) in a Knowledge Graph \( KG = (V, E, R,  \ell_V) \) is \textit{transitive} if for any \( x, y, z \in V \):
\[
(x \xrightarrow{r} y) \wedge (y \xrightarrow{r} z) \Rightarrow (x \xrightarrow{r} z)
\]
\end{definition}
\begin{definition}[Path]
Given a relationship \( r \in R \), a \textit{path} \( p \) from node \( x \) to node \( z \) is defined as a finite sequence of edges:
\[
p = (x_0 \xrightarrow{r} x_1 \xrightarrow{r} x_2 \xrightarrow{r} \dots \xrightarrow{r} x_n)
\]
such that \( x_0 = x \), \( x_n = z \), and each edge \( (x_i \xrightarrow{r} x_{i+1}) \in E \) for \( 0 \leq i < n \). The number of hops in \( p \), denoted \( h(p) \), is equal to \( n \).
\end{definition}
\begin{definition}[Strength of a path]
The \textit{strength} of a specific path \( p \) from node \( x \) to node \( z \), with respect to a transitive relationship \( r \), is defined as:
\[
S_p(x, z, r) = f(h(p))
\]
where \( h(p) \) is the number of hops in the path, and \( f: \mathbb{N} \rightarrow [0, 1] \) is a function that models the attenuation of influence with increasing path length (e.g., due to uncertainty, information decay, or relevance).
\end{definition}
\begin{definition}[Scalable transitive relationship]
    A transitive relationship \( r \in R \) is \textit{scalable} if there exists a function \( S(x, z, r): \mathbb{N} \rightarrow [0, 1] \) assigning a strength to the transitive relationship between nodes \( x \) and \( z \), based on all paths \( p \in P(x, z, r) \), where \( P(x, z, r) \) is the set of paths from \( x \) to \( z \) composed only of relations \( r \). 
     The overall strength is:
\begin{equation}
\label{eq:aggregated_score}
S(x, z, r) = \mathcal{A}(\{ S_p(x, z, r) \mid p \in P(x, z, r) \})
\end{equation}
where \( \mathcal{A} \) is an aggregation function (e.g., max, avg, sum). The relation propagates if:
\[
S(x, z, r) > \tau
\]
with \( \tau \) being a context-dependent threshold.
\end{definition}

In the following Subsections we model the impact of KC respectively on Node2Vec and GrpahSAGE algorithms for generating GEs. 
\subsection{Impact of KC on Node2Vec Embeddings}
In this subsection, we formally describe the Node2Vec algorithm and examine the impact of KC, with a focus on how inferred transitive relationships alter node traversal probabilities.
\begin{definition}[Transition probability]
Given a Knowledge Graph \( KG = (V, R, E, \ell_V) \), let \( c = (v_1, v_2, \ldots, v_k) \) be a path. The transition probability of moving from node \( v = c_{i-1} \) to a neighbor \( x = c_i \) is defined as:
\begin{equation}
P(c_i = x \mid c_{i-1} = v) = \pi_{xv} = \frac{\alpha_{pq}(t,x) \cdot w_{vx}}{Z},
\end{equation}
\noindent
where \(w_{vx}\) is the cumulative weight of the edges between nodes \(v\) and \(x\), \(Z\) is a normalization constant, and \(\alpha_{pq}(t,x)\) is a search bias defined as:
\[
\alpha_{pq}(t,x) =
\begin{cases}
\frac{1}{p} & \text{if } d_{tx} = 0 \ , \\
1           & \text{if } d_{tx} = 1 \,, \\
\frac{1}{q} & \text{if } d_{tx} = 2 \,
\end{cases}
\]
\noindent
with \(d_{tx}\) being the shortest path distance between node \(t\) (the node visited before \(v\)) and node \(x\), \(p\) being the return parameter and \(q\) being the in-out parameter.
\end{definition}

\begin{definition}[Node2Vec on Knowledge Graphs]
Given a Knowledge Graph \( KG = (V, R, E, \ell_V) \), Node2Vec \cite{b78} is a method that learns low-dimensional vector embeddings \( \phi : V \rightarrow \mathbb{R}^d \) for nodes by simulating biased random walks over the graph structure.
A sequence of nodes is generated via a second-order random walk defined by the transition probability:
\begin{equation}
 P(c_{i} = x \mid c_{i-1} = v) =
\begin{cases}
\frac{1}{p} & \text{if } x = c_{i-2} \\
1           & \text{if } d(x, c_{i-2}) = 1 \\
\frac{1}{q} & \text{otherwise}
\end{cases}   
\end{equation}
where:
\begin{itemize}
    \item \( c_i \) denotes the node $c$ at position \( i \) in the walk;
    \item \( p \) is the return parameter controlling the likelihood of revisiting the previous node;
    \item \( q \) is the in-out parameter controlling the likelihood of exploring further away nodes;
    \item \( d(x, c_{i-2}) \) is the shortest path distance between node \( x \) and the node before the current one.
\end{itemize}
\end{definition}

\begin{definition}[Influence of KC in Node2Vec]
Let \( KG = (V, E, R, \ell_V) \) be a Knowledge Graph, and let \( u \in V \) be a node. Let \( \mathcal{W} \) be a Node2Vec sampling process which generates a collection of walks: 
\begin{equation}
\mathcal{P} = \{ P_1, P_2, \dots, P_k \},
\end{equation}
where each walk \( P_i = (v_{i,1}, v_{i,2}, \dots, v_{i,T}) \) is a sequence of \( T \) steps.
The probability that node \( u \) is visited during these walks is given by:
\begin{equation}
    \Pr_{\mathcal{W}}(u) = \frac{\sum_{i=1}^{k} \sum_{t=1}^{T} \mathbf{1}_{\{v_{i,t} = u\}}}{k \cdot T},
\end{equation}
where $k$ is the number of random walks and \( \mathbf{1}_{\{v_{i,t} = u\}} \) is the indicator function that is 1 if \( v_{i,t} = u \) and 0 otherwise.
\end{definition}

\begin{lemma}\label{lemma:node2vec}
Let \( KG = (V, E, R, \ell_V)) \) be a Knowledge Graph and \( KG' = (V, E \cup E_{KC}, R, \ell_V ) \) the Knowledge Graph resulting from a KC step that infers transitive relationships. Let \(u \in V\) be a node, let \( \mathcal{W} \) be a Node2Vec algorithm executed on \( KG \), and let \( \mathcal{W}' \) be the same Node2Vec algorithm executed on  \( KG' \). 
Then, \[ \Pr_{\mathcal{W}'}(u) \geq \Pr_{\mathcal{W}}(u) .\]  
\end{lemma}
\begin{proof}[Proof of Lemma \ref{lemma:node2vec}]
The probability of node \(u\) to be traversed during the execution of \(\mathcal{W}'\) is:
\begin{equation}
    \Pr_{\mathcal{W}'}(u) = \frac{\sum_{i=1}^{k} \sum_{t=1}^{T} \mathbf{1}_{\{v'_{i,t} = u\}}}{k \cdot T}.
\end{equation}
Since KC can add new edges to the existing set, we have \( E' \supseteq E \), implying that the space of possible walks in \( KG' \) includes all those from \( KG \) and potentially more.
Hence, for any node \( u \in V \), it holds that:
\begin{equation}
\Pr_{\mathcal{W}'}(u) \geq \Pr_{\mathcal{W}}(u).
\end{equation}

\end{proof}


\subsection{Effects of Transitive relationships on Node2Vec}
The integration of a transitive strength coefficient \( S(v, u, r) \in [0, 1] \), enables a weighted formulation of the transition probabilities in Node2Vec. In this formulation, the transition probability from node \( v \) to a neighbor \( u \) is redefined as:
\begin{equation}
P(u \mid v) = S(v, u, r) \cdot \pi_{vu}
\end{equation}
The strength coefficient acts as a multiplicative weight, modifying the walk probabilities according to the transitive semantics encoded in the graph.

\subsection{Impact of KC on GraphSAGE Embeddings}
In this subsection, we formally present the GraphSAGE algorithm and investigate the impact of KC, focusing on how the addition of inferred transitive relationships influences node aggregation.

\begin{definition}[GraphSAGE on Knowledge Graphs]
Given a Knowledge Graph \( KG = (V, R, E, \ell_V) \), GraphSAGE (Graph Sample and Aggregate) is an inductive framework that learns
embeddings
by aggregating information from the local neighborhood of each node.
At each layer \( k \), the representation of a node \( v \in V \), denoted as \( h_v^{(k)} \in \mathbb{R}^d \), is updated as:
\begin{multline}
h_v^{(k)} = \sigma\Big( W^{(k)} \cdot 
\text{AGG}^{(k)}\big( \{ h_u^{(k-1)} \mid (u, r, v) \in E \} \\
\cup \{ h_v^{(k-1)} \} \big) \Big)
\end{multline}
where:
\begin{itemize}
    \item \( h_v^{(0)} \) is the initial feature vector of node \( v \);
    \item \( \text{AGG}^{(k)} \) is a permutation-invariant aggregator function;
    \item \( W^{(k)} \) is a trainable weight matrix at layer \( k \);
    \item \( \sigma \) is a non-linear activation function.
\end{itemize}
GraphSAGE supports inductive learning by enabling generalization to unseen nodes through learned aggregation functions over sampled neighborhoods\cite{b77}.
\end{definition}

\begin{definition}[Influence of KC in GraphSAGE]
Let \( KG = (V, E, R, \ell_V) \) be a Knowledge Graph, and let \( u \in V \) be a node. Let \( \mathcal{S} \) be the GraphSAGE sampling process used to collect neighborhoods up to depth \( K \) for each node \( v \). Denote \( N^{(k)}(v) \) the \( k \)-hop sampled neighborhood of node \( v \), and \( h_u^{(k)} \in \mathbb{R}^d \) the representation of node \( u \) at layer \( k \).  
We define the \emph{aggregation influence} of node \( u \) as:
\begin{equation}
\mathcal{I}_{\mathcal{S}}(u) = \frac{1}{|V| \cdot K} \sum_{v \in V} \sum_{k=1}^{K} \mathbf{1}_{\{u \in N^{(k)}(v)\}},
\end{equation}
where \( \mathbf{1}_{\{u \in N^{(k)}(v)\}} \) is 1 if node \( u \) is sampled in the \( k \)-hop neighborhood of \( v \), and 0 otherwise.
\end{definition}
\begin{lemma}\label{lemma:graphsage}
Let \( KG = (V, E, R, \ell_V) \) be a Knowledge Graph and let \( KG' = (V, E \cup E_{KC}, R, \ell_V) \) be the Knowledge Graph obtained by applying a Knowledge Completion (KC) step that infers transitive relationships and adds edges. Let \( \mathcal{S} \) be the GraphSAGE sampling strategy on \( KG \), and \( \mathcal{S}' \) the same strategy on \( KG' \). Then, for any node \( u \in V \), it holds that:
\begin{equation}
\mathcal{I}_{\mathcal{S}'}(u) \geq \mathcal{I}_{\mathcal{S}}(u).
\end{equation}
\end{lemma}

\begin{proof}[Proof of Lemma \ref{lemma:graphsage}]
Knowledge Completion (KC) adds edges to the graph, so the updated neighborhood function \( N'^{(k)}(v) \) in \( KG' \) satisfies:
\begin{equation}
N^{(k)}(v) \subseteq N'^{(k)}(v),
\end{equation}
for each node \( v \in V \) and each aggregation layer \( k \in \{1, \dots, K\} \), assuming the same sampling budget per layer. Therefore, any node \( u \) that appears in \( N^{(k)}(v) \) under \( \mathcal{S} \), will also appear under \( \mathcal{S}' \), with the possibility that additional appearances occur due to the expanded edge set.

This implies that:
\begin{equation}
\mathbf{1}_{\{u \in N^{(k)}(v)\}} \leq \mathbf{1}_{\{u \in N'^{(k)}(v)\}} \quad \forall v \in V, \forall k.
\end{equation}

Summing over all \( v \) and all \( k \), and normalizing:
\begin{equation}
\mathcal{I}_{\mathcal{S}'}(u)  \geq \mathcal{I}_{\mathcal{S}}(u).
\end{equation}

Hence, the influence of node \( u \) in the aggregation process increases (or remains the same) when transitive knowledge is added.
\end{proof}

\subsection{Effects of Transitive relationships on GraphSAGE}
The introduction of a transitive strength coefficient \( S(v, u, r) \in [0,1] \), enables a weighted formulation of the neighborhood aggregation process in the GraphSAGE framework. 
The standard GraphSAGE update rule is extended as follows:
\begin{multline}
h_v^{(k)} = \sigma\Big( W^{(k)} \cdot 
\text{AGG}^{(k)}\big( \{ S(v, u, r) \cdot h_u^{(k-1)} \mid (u, r, v) \in E \} \\
\cup \{ h_v^{(k-1)} \} \big) \Big)
\end{multline}
where:
\( S(v, u, r) \) is a weighting factor that modifies the contribution of each neighbor \( u \) in the computation of the embedding for node \( v \) at layer \( k \).
\section{Experimental Evaluation on Temporal Face-to-Face Contact Networks}
To assess the impact of our architecture on GEs and its influence propagation in  interaction networks, we designed a set of experiments using Node2Vec and GraphSAGE. Experiments were performed in Neo4j 5.24.0 GDB\footnote{neo4j.com}, using the Graph Data Science (GDS) library 2.12.0\footnote{neo4j.com/docs/graph-data-science} to compute measures.
Results from our enriched system were then compared with those from the standard Neo4j-GML pipeline.

We used a real-world dataset which describes face-to-face interactions between individuals within an office building in France. The dataset, originally collected for epidemiological analysis, provides high-resolution temporal contact data between pairs of individuals\cite{b90}.
It is structured in two files, in particular the first one contains temporal edges of the form \texttt{t i j}, indicating that individuals $i$ and $j$ were in close-range physical proximity during a 20-second interval ending at time $t$, while the second file maps each individual ID $i$ to a department $D_i$.
The dataset was converted into a $KG = (V, E, R,\ell_V)$, where $V$ is the set of users and departments, $R = \{\texttt{HAS\_CONTACT\_WITH}, \texttt{IS\_PART\_OF}\}$, and $E \subseteq V \times R \times V$ the set of edges. Each $\texttt{HAS\_CONTACT\_WITH}$ edge has a parameter $t \in \mathbb{R}^+$ for total contact time and another one for the timestamp. Each node $v \in V$ is associated with: contagion probability $cp(v)$, degree centrality $\deg(v)$, total contact time $T(v)$ and average contact time $\bar{T}(v)$ (Fig. \ref{fig:kg_schema}).

\begin{figure}[ht!]
    \centering
    \includegraphics[width=0.85\linewidth]{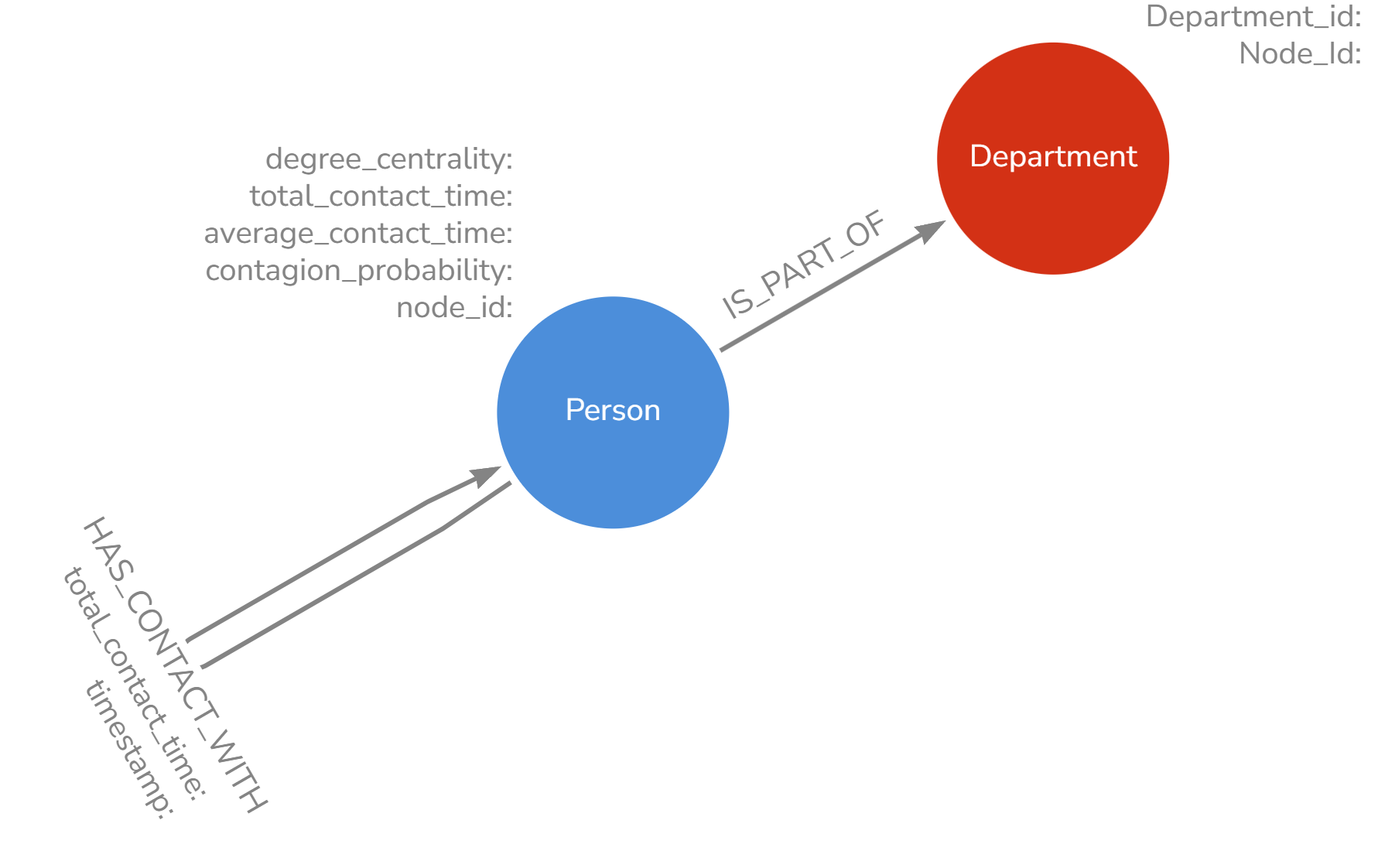}
    \caption{KG Schema.}
    \label{fig:kg_schema}
\end{figure}

\subsection{First KC Step: Handling Indirect Transitive Contacts}
Since the dataset is designed to describe relationships between pairs of individuals, it often occurs that a person $x$ is in contact with a person $y$, and at the same timestamp, $x$ is also in contact with a person $z$. However, the original dataset does not include a direct relationship between $y$ and $z$, even though it is evident that they are in the same room or present at the same location at the same time. This represents simple transitive relationships, which we introduce in order to enrich the semantic structure of the original dataset (Fig. \ref{fig:inferredl}).

\begin{figure}[ht!]
    \centering
    \includegraphics[width=0.75\linewidth]{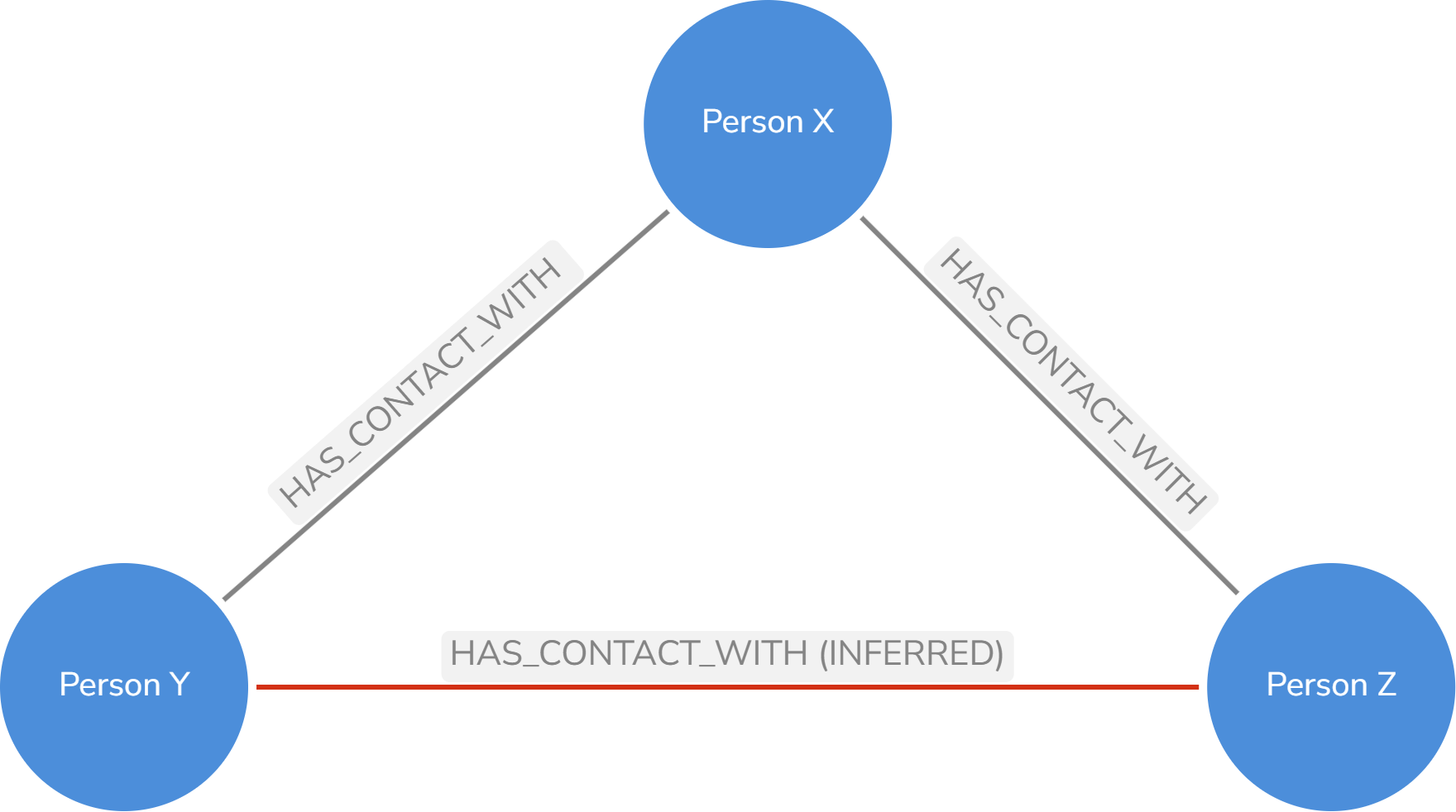}
    \caption{Inferred Contacts.}
    \label{fig:inferredl}
\end{figure}

To do that, we implemented a transitive closure algorithm at each timestamp using a modified BFS. By doing that, we identified connected components of individuals at each 20-second interval and generated all pairwise transitive links among individuals, even if not directly recorded. The process begins by grouping all interactions by timestamp $t$, then building a contact Knowledge Graph $KG_t$ where each edge $(i,j)$ represents a contact in the interval $[t{-}20, t]$. A BFS is applied to $KG_t$ to identify connected components. For each component, all possible pairs $(i,j)$ with $i \neq j$ are generated to perform transitive completion. At the end, for each moment in time, if person $i$ was connected to person $k$ and $k$ to person $j$, then a link between $i$ and $j$ is inferred.
Using the above method, we build two Knowledge Graphs instances:
\begin{itemize}
    \item $KG_{raw}$, which is the original contact graph with direct edges only;
    \item $KG_{KC}$, the enriched graph with transitive edges inferred using KC.
\end{itemize}

The number of contacts between users changed from 1694 to 1882, which is an increase of 11.8\% (Fig \ref{fig:increase}).
\begin{figure}[ht!]
    \centering
    \includegraphics[width=0.83\linewidth]{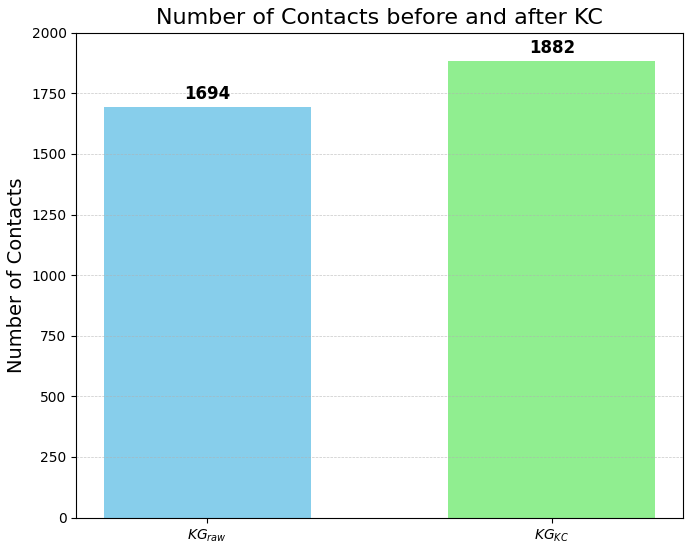}
    \caption{Contact growth between $KG_{\mathrm{raw}}$ and $KG_{\mathrm{KC}}$.}
    \label{fig:increase}
\end{figure}
\subsection{Second KC Step: Calculating the probability of COVID-19 spread via Transitive Scalable Relationships}
 We adopted a probabilistic model inspired by the dynamics of indoor COVID contagion\cite{b91} as strength function, to evaluate the effect of indirect contacts. The infection probability $P_B$ for a person $B$ exposed to an infected individual $A$ for $t$ seconds is defined by:
\begin{equation}
P_B = P_A  - e^{-\beta \cdot t},
\end{equation}
where $\beta$ is the transmissibility coefficient, set to 0.01 to reflect high-risk indoor contact scenarios, and \( t \) the contact time. 
The overall contagion probability for a node $v$, indirectly connected to an infected node via multiple exposures, is aggregated as:
\begin{equation}
P_v^{\text{total}} = 1 - \prod_{i=1}^{k} \left(1 - \left(P_i - e^{-\beta \cdot t_i}\right)\right),
\end{equation}
where $P_v^{\text{total}}$ is the cumulative probability that at least one of the $k$ paths successfully transmits the infection and $P_i$ is the infection probability from each path with $\tau$ threshold set to 0.2. 
The decay function \( f(h) \) is modeled as \( e^{-\beta \cdot t} \). The path strength \( S_p(x, z, r) \) becomes \( P_i - e^{-\beta \cdot t_i} \), so Eq. \ref{eq:aggregated_score} can be written as:
\begin{equation}
    S(x, z, r) = 1 - \prod_{i=1}^{k} \left(1 - S_p(x, z, r)_i \right)
\end{equation}

\subsection{Effect of KC on most influent nodes}
\begin{definition}[PageRank on Knowledge Graphs]
PageRank is an algorithm that measures the importance of nodes by distributing scores through incoming links, preferring those connected to already influential ones. Formally, given a Knowledge Graph \( KG = (V, R, E, \ell_V) \), the PageRank score \( PR(v) \) of a node \( v \in V \) is defined as:
\[
PR(v) = \frac{1 - \alpha}{|V|} + \alpha \sum_{(u, r, v) \in E} \frac{PR(u)}{|\{ (u, r', w) \in E \}|}
\]
where:
\begin{itemize}
    \item \( \alpha \in (0,1) \) is the damping factor, set to \( 0.85 \) by default;
    \item \( (u, r, v) \in E \) indicates that node \( u \) is connected to node \( v \) via a relationship \( r \in R \);
    \item The denominator counts the number of outgoing edges from node \( u \), regardless of relation type\cite{b93}.
\end{itemize}
\end{definition}
We simulated the introduction of infected individuals to evaluate the effect of KC on influence dynamics. For both $KG_{raw}$ and $KG_{KC}$, we selected as initial seeds the top-5 nodes with highest Pagerank score. Starting from contacts, we then calculated the new infection probabilities.

\begin{table}[ht!]
\centering
\rowcolors{2}{gray!10}{white}
\begin{tabular}{|c|l|c|l|c|}
\hline
\rowcolor{gray!30}
\textbf{Rank} & \textbf{Without KC} & \textbf{PageRank} & \textbf{With KC} & \textbf{PageRank} \\
\hline
1  & Person 804 & 2.3536 & Person 804 & 2.0927 \\ 
2  & Person 311 & 2.0100 & Person 311 & 2.0202 \\ 
3  & \cellcolor{cyan!25}Person 95  & 1.7348 & \cellcolor{blue!20}Person 134 & 1.7626 \\
4  & \cellcolor{purple!25}Person 80  & 1.6639 & \cellcolor{teal!25}Person 223 & 1.7056 \\
5  & \cellcolor{blue!20}Person 134 & 1.6294 & \cellcolor{purple!25}Person 80  & 1.7021 \\
6  & \cellcolor{green!25}Person 63  & 1.5805 & \cellcolor{cyan!25}Person 95  & 1.5659 \\
7  & \cellcolor{teal!25}Person 223 & 1.5440 & \cellcolor{yellow!25}Person 875 & 1.4803 \\
8  & \cellcolor{yellow!25}Person 875 & 1.4645 & \cellcolor{green!25}Person 63  & 1.4303 \\
9  & Person 826 & 1.4387 & Person 662 & 1.4191 \\ 
10 & Person 123 & 1.4329 & Person 194 & 1.3944 \\ 
\hline
\end{tabular}
\vspace{0.5em}
\caption{Top 10 nodes PageRank values with and without KC.}
\label{tab:pagerank_kc_comparison}
\end{table}
The modification of the KC step radically alters the top PageRank nodes, indicating a complete change in the most important elements (Tab. \ref{tab:pagerank_kc_comparison}).

\subsection{Effects of KC on GEs generation}
We configured the GraphSAGE model to sample up to 25 neighbors in the first convolutional layer and 10 in the second, using a mean aggregator to integrate neighborhood information, with ReLU as activation function.
On the other hand, we configured Node2Vec with 10 random walks of length 80, a context window of 10, and default parameters ($p = q = 1$).
In both cases, we computed 16-dimensional embeddings. Then, we applied  Principal Component Analysis (PCA) to project them into two dimensions to visualize the results.

In the GraphSAGE model, the 2D PCA projection (Fig.~\ref{fig:GraphSAGE embeddings comparison}) shows a noticeable shift in the vector space. Embeddings from our improved pipeline (blue) have a broader spatial distribution and increased variance, indicating that the model captures additional semantic information from the inferred edges, which leads to more different embeddings. In contrast, the normal workflow (red) shows more compact clustering, suggesting a more limited representation of node interactions.

\begin{figure}[ht!]
    \centering
    \includegraphics[width=1\linewidth]{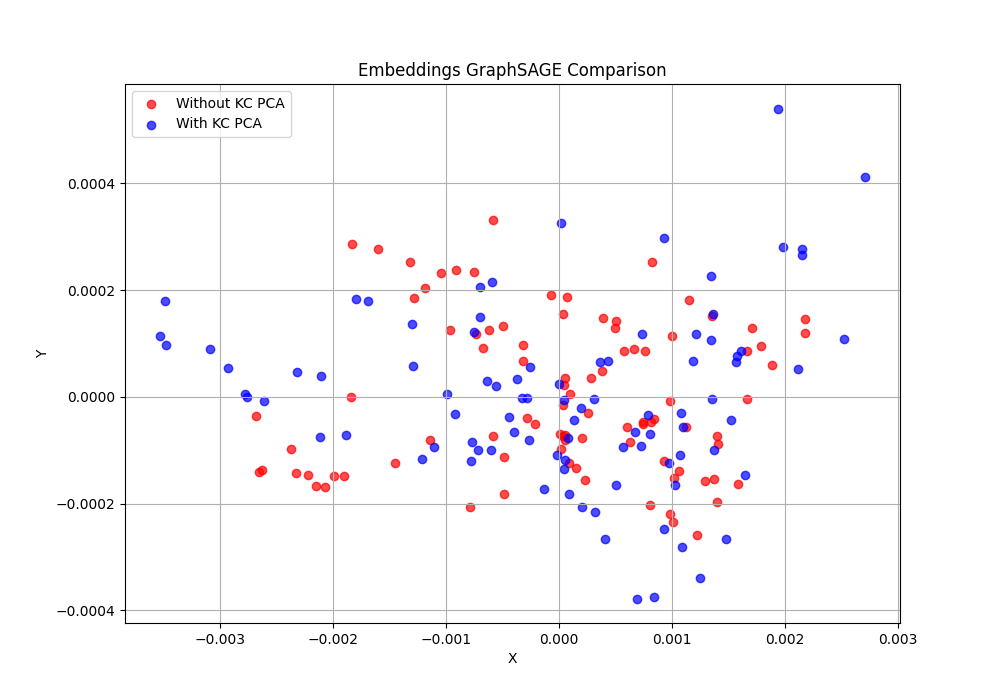}
    \caption{GraphSAGE embeddings comparison.}
    \label{fig:GraphSAGE embeddings comparison}
\end{figure}

Applying Node2Vec we obtained noticeable changes in the vector space (Fig.~\ref{fig:Node2Vec embeddings comparison}). In contrast to GraphSAGE, embeddings with KC (blue) show  less highly isolated nodes. These results suggest that the inferred edges help previously disconnected elements, enhancing the overall structure of the learned representations.
\begin{figure}[ht!]
    \centering
    \includegraphics[width=1\linewidth]{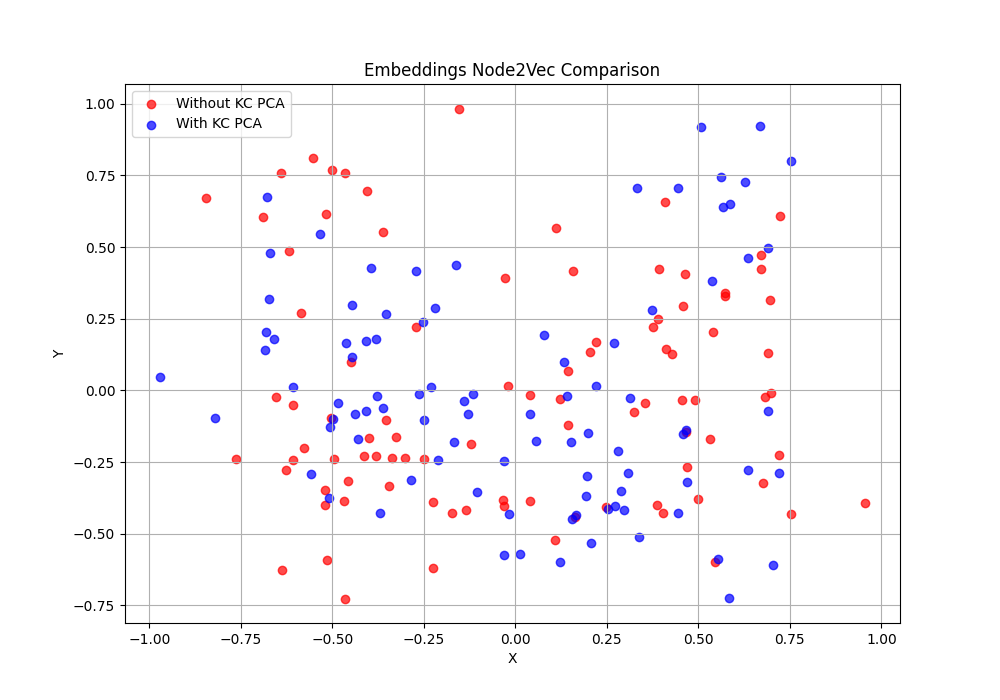}
    \caption{Node2Vec embeddings comparison.}
    \label{fig:Node2Vec embeddings comparison}
\end{figure}

The structural modifications induced by KC on the \( KG \) resulted in measurable alterations in the learned node embeddings. To quantify this impact, we computed the Euclidean distance between the embeddings of each corresponding node in the original graph \( KG_{\text{raw}} \) and the completed graph \( KG_{\text{KC}} \). Formally, for each node \( v_i \in V \), we measured:
\[
d_i = \left\| \mathbf{v}_i^{\text{raw}} - \mathbf{v}_i^{\text{KC}} \right\|_2
\]
where \( \mathbf{v}_i^{\text{raw}} \) and \( \mathbf{v}_i^{\text{KC}} \) denote the embedding vectors of node \( v_i \) in \( KG_{\text{raw}} \) and \( KG_{\text{KC}} \), respectively.

The average Euclidean distance over all nodes was found to be \( 0.83 \) for embeddings produced by Node2Vec, and \( 0.041 \) for those generated by GraphSAGE. While these values offer a preliminary quantitative assessment of the embedding shifts caused by KC, further investigation is needed to analyze the nature and distributional characteristics of these changes across the embedding space.

Our pipeline improved GraphSAGE by enabling a more meaningful aggregation of node relationships, leading to a refined understanding of the graph data. For Node2Vec, the enhancements resulted in better random walk representations, improving path connections across the network, leading to more accurate and informative embeddings in both models.


\section{Conclusions and Future Works}
In this work, we demonstrated how integrating a  KC phase before KR step can profoundly influence the behavior and results of GML pipelines. By enriching graph structures with inferred transitive relationships, we transformed the neighborhood topologies, modifying the main embedding generation techniques. The experimental evaluation showed that our new architecture reshaped node proximities, centrality dynamics  and enabled the generation of embeddings that better captured hidden semantic structures within the data. These results confirm that KC before KR is not just  a data augmentation phase, but a transformative process that redefines the embedding space, leading to more meaningful and accurate node representations.

As future work, we plan to study how KC-enriched embeddings impact downstream supervised tasks, with particular focus on classification and regression performances in various Graph Neural Network (GNN) architectures. We aim to assess whether the additional semantic structure introduced by the KC step translates into better generalization capabilities and higher accuracy. Furthermore, we intend to conduct an in-depth analysis on the scalability of the proposed KC methodology and explore performance optimizations to ensure its applicability to increasingly large and complex graph datasets.
\section*{Acknowledgment}
Rosario Napoli is a PhD student enrolled in the National PhD
in Artificial Intelligence, XL cycle, course on Health and
life sciences. This work has been partially funded by the the Italian Ministry of Health, Piano Operativo Salute (POS) trajectory 4 “Biotechnology, bioinformatics and pharmaceutical development”, through the Pharma-HUB Project "Hub for the repositioning of drugs in rare diseases of the nervous system in children" (CUP J43C22000500006), the Italian Ministry of Health, Piano Operativo Salute (POS) trajectory 2 “eHealth, diagnostica avanzata, medical device e mini invasività” through the project ``Rete eHealth: AI e strumenti ICT Innovativi orientati alla Diagnostica Digitale (RAIDD)''(CUP J43C22000380001), the “SEcurity and RIghts in the CyberSpace (SERICS)” partnership (PE00000014), under the MUR National Recovery and Resilience Plan funded by the European Union – NextGenerationEU. In particular, it has been supported within the SERICS partnership through the projects FF4ALL (CUP D43C22003050001) and SOP (CUP H73C22000890001), the Horizon Europe NEUROKIT2E project (Grant Agreement 101112268). 

\bibliographystyle{unsrt}
\bibliography{bib}

@article{b6,
    title = {Big Data Management Challenges, Approaches, Tools and their limitations},
    author = {M. Adiba et al.},
    journal = {None},
    year = {2016},
    doi = {None},
    url ={https://www.semanticscholar.org/paper/73947610e22e6d0c3366bb027529a39d7a1879b5}
}

@article{b12,
author = {Agrawal, Smita and Patel, Atul},
year = {2016},
month = {02},
pages = {33-39},
title = {A Study on Graph Storage Database of NOSQL},
volume = {5},
journal = {International Journal on Soft Computing, Artificial Intelligence and Applications},
doi = {10.5121/ijscai.2016.5104}
}

@article{b32,
  author = {Hogan, Aidan et al.},
  title = {Knowledge Graphs},
  journal = {Association for Computing Machinery},
  year = {2021},
  doi = {10.1145/3447772}
}

@article{b33,
  author = {Nguyen, H. L. and Vu, D. T. and Jung, J. J.},
  title = {Knowledge graph fusion for smart systems: A Survey},
  journal = {Information Fusion},
  volume = {61},
  pages = {56--70},
  year = {2020},
  doi = {10.1016/j.inffus.2020.03.014}
}

@article{b53,
  author       = {Maciej Besta and Patrick Iff and Florian Scheidl and Kazuki Osawa and Nikoli Dryden and Michal Podstawski and Tiancheng Chen and Torsten Hoefler},
  title        = {Neural Graph Databases},
  journal      = {LOG IN},
  year         = {2022},
  doi          = {10.48550/arXiv.2209.09732},
  url          = {https://doi.org/10.48550/arXiv.2209.09732}
}

@article{b72,
  author    = {Makarov, Ilya and Kiselev, Dmitrii and Nikitinsky, Nikita and Šubelj, Lovro},
  title     = {Survey on graph embeddings and their applications to machine learning problems on graphs},
  journal   = {PeerJ Computer Science},
  year      = {2021},
  publisher = {PeerJ, Inc.},
  doi       = {10.7717/peerj-cs.357},
  url       = {https://doi.org/10.7717/peerj-cs.357}
}

@article{b74,
  author    = {Jiang, Xiaodong and Zhu, Ronghang and Ji, Pengsheng and Li, Sheng},
  title     = {Co-Embedding of Nodes and Edges With Graph Neural Networks},
  journal   = {IEEE Transactions on Pattern Analysis and Machine Intelligence},
  year      = {2020},
  publisher = {IEEE Computer Society},
  doi       = {10.1109/TPAMI.2020.3029762},
  url       = {https://doi.org/10.1109/TPAMI.2020.3029762}
}

@article{b75,
  author    = {Wang, Changping and Wang, Chaokun and Wang, Zheng and Ye, Xiaojun and Yu, Philip S.},
  title     = {Edge2vec: Improving Representation Learning Using Edge Semantics},
  journal   = {ACM Transactions on Knowledge Discovery from Data (TKDD)},
  year      = {2020},
  publisher = {Association for Computing Machinery},
  doi       = {10.1145/3391298},
  url       = {https://doi.org/10.1145/3391298}
}

@article{b77,
  author    = {Shima Khoshraftar and Aijun An},
  title     = {A Survey on Graph Representation Learning Methods},
  journal   = {ACM Computing Surveys},
  year      = {2023},
  publisher = {Association for Computing Machinery},
  doi       = {10.1145/3633518},
  url       = {https://doi.org/10.1145/3633518}
}

@article{b78,
  author    = {Renming Liu and Arjun Krishnan},
  title     = {PecanPy: a fast, efficient and parallelized Python implementation of \textit{node2vec}},
  journal   = {Bioinformatics},
  volume    = {37},
  number    = {22},
  pages     = {3978--3980},
  year      = {2021},
  publisher = {Oxford University Press},
  doi       = {10.1093/bioinformatics/btab202},
  url       = {https://doi.org/10.1093/bioinformatics/btab202}
}

@inproceedings{b79,
  author    = {Maya Kapoor and Joshua Melton and Michael Ridenhour and S. Krishnan and Thomas Moyer},
  title     = {{PROV-GEM}: Automated Provenance Analysis Framework using Graph Embeddings},
  booktitle = {2021 20th IEEE International Conference on Machine Learning and Applications (ICMLA)},
  year      = {2021},
  pages     = {1536--1541},
  publisher = {IEEE},
  doi       = {10.1109/ICMLA52953.2021.00273},
  url       = {https://doi.org/10.1109/ICMLA52953.2021.00273}
}

@article{b80,
  author    = {Yue Deng},
  title     = {Recommender Systems Based on Graph Embedding Techniques: A Review},
  journal   = {IEEE Access},
  volume    = {10},
  pages     = {53793--53806},
  year      = {2022},
  publisher = {IEEE},
  doi       = {10.1109/ACCESS.2022.3174197},
  url       = {https://doi.org/10.1109/ACCESS.2022.3174197}
}

@article{b81,
  author    = {Mengyuan Lee and Guanding Yu and Geoffrey Ye Li},
  title     = {Graph Embedding-Based Wireless Link Scheduling With Few Training Samples},
  journal   = {IEEE Transactions on Wireless Communications},
  volume    = {20},
  number    = {4},
  pages     = {2542--2554},
  year      = {2021},
  publisher = {IEEE},
  doi       = {10.1109/TWC.2020.3040983},
  url       = {https://doi.org/10.1109/twc.2020.3040983}
}

@article{b84,
  author    = {Wenhao Yu and Chenguang Zhu and Zaitang Li and Zhiting Hu and Qingyun Wang and Heng Ji and Meng Jiang},
  title     = {A Survey of Knowledge-enhanced Text Generation},
  journal   = {ACM Computing Surveys},
  volume    = {55},
  number    = {7},
  pages     = {1--38},
  year      = {2022},
  publisher = {Association for Computing Machinery},
  doi       = {10.1145/3512467},
  url       = {https://doi.org/10.1145/3512467}
}

@inproceedings{b85,
  title     = {Knowledge-Enhanced Hierarchical Graph Transformer Network for Multi-Behavior Recommendation},
  author    = {Xia, Lianghao and Huang, Chao and Xu, Yong and Dai, Peng and Zhang, Xiyue and Yang, Hongsheng and Pei, Jian and Bo, Liefeng},
  booktitle = {Proceedings of the AAAI Conference on Artificial Intelligence},
  year      = {2021},
  volume    = {35},
  pages     = {4462--4470},
  doi       = {10.1609/aaai.v35i5.16576},
  publisher = {Association for the Advancement of Artificial Intelligence},
  url       = {https://doi.org/10.1609/aaai.v35i5.16576}
}

@inproceedings{b86,
  title     = {DiriE: Knowledge Graph Embedding with Dirichlet Distribution},
  author    = {Wang, Feiyang and Zhang, Zhongbao and Sun, Li and Ye, Junda and Yan, Yang},
  booktitle = {Proceedings of the ACM Web Conference 2022 (The Web Conference)},
  year      = {2022},
  pages     = {2284--2293},
  doi       = {10.1145/3485447.3512028},
  publisher = {Association for Computing Machinery},
  url       = {https://doi.org/10.1145/3485447.3512028}
}

@inproceedings{b87,
  title     = {Knowledge Embedding Based Graph Convolutional Network},
  author    = {Yu, Donghan and Yang, Yiming and Zhang, Ruohong and Wu, Yuexin},
  booktitle = {Proceedings of the 2021 ACM Web Conference},
  year      = {2021},
  doi       = {10.1145/3442381.3449925},
  url       = {https://doi.org/10.1145/3442381.3449925},
  publisher = {Association for Computing Machinery}
}

@inproceedings{b88,
  title     = {Bipartite Graph Embedding via Mutual Information Maximization},
  author    = {Cao, Jiangxia and Lin, Xixun and Guo, Shu and Liu, Luchen and Liu, Tingwen and Wang, Bin},
  booktitle = {Proceedings of the 14th ACM International Conference on Web Search and Data Mining (WSDM)},
  year      = {2021},
  pages     = {79--87},
  doi       = {10.1145/3437963.3441783},
  url       = {https://doi.org/10.1145/3437963.3441783},
  publisher = {Association for Computing Machinery}
}

@inproceedings{b89,
  title     = {Simple and Efficient Heterogeneous Graph Neural Network},
  author    = {Yang, Xiaocheng and Yan, Mingyu and Pan, Shirui and Ye, Xiaochun and Fan, Dongrui},
  booktitle = {Proceedings of the AAAI Conference on Artificial Intelligence},
  year      = {2023},
  volume    = {37},
  pages     = {10756--10764},
  doi       = {10.1609/aaai.v37i9.26283},
  url       = {https://doi.org/10.1609/aaai.v37i9.26283},
  publisher = {Association for the Advancement of Artificial Intelligence}
}

@article{b90,
author = {GÉNOIS,MATHIEU et al.},
title = {Data on face-to-face contacts in an office building suggest a low-cost vaccination strategy based on community linkers},
journal = {Network Science},
volume = {3},
issue = {03},
month = {9},
year = {2015},
issn = {2050-1250},
pages = {326--347},
numpages = {22},
doi = {10.1017/nws.2015.10},
URL = {http://journals.cambridge.org/article_S2050124215000107},
}

@article{b91,
  author    = {Bazant, Martin Z. and Bush, John W. M.},
  title     = {A guideline to limit indoor airborne transmission of COVID-19},
  journal   = {Proceedings of the National Academy of Sciences},
  volume    = {118},
  number    = {17},
  pages     = {e2018995118},
  year      = {2021},
  doi       = {10.1073/pnas.2018995118},
  pmid      = {33858987},
  pmcid     = {PMC8092463}
}

@article{b92,
  author={Xia, Feng and Sun, Ke and Yu, Shuo and Aziz, Abdul and Wan, Liangtian and Pan, Shirui and Liu, Huan},
  journal={IEEE Transactions on Artificial Intelligence}, 
  title={Graph Learning: A Survey}, 
  year={2021},
  volume={2},
  number={2},
  pages={109-127},
  doi={10.1109/TAI.2021.3076021}}

@inproceedings{b93,
  title={The PageRank Citation Ranking : Bringing Order to the Web},
  author={Lawrence Page and Sergey Brin and Rajeev Motwani and Terry Winograd},
  booktitle={The Web Conference},
  year={1999},
  url={https://api.semanticscholar.org/CorpusID:1508503}
}

\end{document}